\documentclass[a4paper]{article}

\usepackage{amsmath}
\usepackage{amssymb}
\usepackage{amsthm}
\allowdisplaybreaks
\usepackage{booktabs} 
\usepackage{bm}
\usepackage{bbm}

\usepackage[verbose=true,letterpaper]{geometry}
\AtBeginDocument{
  \newgeometry{
    textheight=9in,
    textwidth=5.5in,
    top=1in,
    headheight=12pt,
    headsep=25pt,
    footskip=30pt
  }
}

\newtheorem{corollary}{Corollary}
\newtheorem{lemma}{Lemma}
\newtheorem{proposition}{Proposition}

\newcommand{\argmin}{\mathop{\mathrm{argmin}}}
\newcommand{\CondCov}[3]{\mathrm{Cov}\left(#1, #2 \,\middle|\, #3\right)}

\newcommand{\CondProb}[2]{\mathbb{P}\left[#1 \,\middle|\, #2\right]}
\newcommand{\CondVar}[2]{\mathrm{Var}\left(#1 \,\middle|\, #2\right)}

\newcommand{\E}[2][]{\mathbb{E}_{#1}\left[#2\right]}
\newcommand{\Prob}[1]{\mathbb{P}\left[#1\right]}
\newcommand{\rank}{\mathrm{rank}}
\newcommand{\sign}{\mathrm{sign}}

\newcommand{\Supp}{\mathrm{Supp}}

\usepackage{xcolor}
\definecolor{blue3}{HTML}{0076BA}
\definecolor{pink3}{HTML}{CB297B}

\usepackage{hyperref}
\hypersetup{
  colorlinks=true,
  citecolor={pink3},
  linkcolor={pink3},
}
\usepackage{natbib}
\usepackage{url}

\begin{document}
\title{Revisiting the Vector Space Model: Sparse Weighted Nearest-Neighbor Method for Extreme Multi-Label Classification}

\author{
  Tatsuhiro Aoshima \\
  Graduate School of Keio University \\
  \texttt{hiro4bbh@keio.jp}
  \and
  Kei Kobayashi, Mihoko Minami \\
  Keio University \\
  \texttt{\{kei,mminami\}@math.keio.ac.jp}
}

\maketitle

\begin{abstract}
Machine learning has played an important role in information retrieval (IR) in recent times.
In search engines, for example, query keywords are accepted and documents are returned in order of relevance to the given query; this can be cast as a multi-label ranking problem in machine learning.
Generally, the number of candidate documents is extremely large (from several thousand to several million); thus, the classifier must handle many labels.
This problem is referred to as extreme multi-label classification (XMLC).
In this paper, we propose a novel approach to XMLC termed the \emph{Sparse Weighted Nearest-Neighbor Method}.
This technique can be derived as a fast implementation of state-of-the-art (SOTA) one-versus-rest linear classifiers for very sparse datasets.
In addition, we show that the classifier can be written as a sparse generalization of a representer theorem with a linear kernel.
Furthermore, our method can be viewed as the vector space model used in IR.
Finally, we show that the Sparse Weighted Nearest-Neighbor Method can process data points in real time on XMLC datasets with equivalent performance to SOTA models, with a single thread and smaller storage footprint.
In particular, our method exhibits superior performance to the SOTA models on a dataset with 3 million labels.
\end{abstract}


\section{Introduction}

In recent times, machine learning has become important for information retrieval (IR).
Search engines are notable examples of this relationship, as input query keywords are used to return documents in order of relevance to the given entry.
A keyword-tagging system assesses the content of a candidate document and returns tagged keywords in order of relevance to the document.
Furthermore, an advertisement system takes viewer and landing page information and returns advertisements in order of their expected click-through rate (CTR; \citet{McMahan_etal2013}).

These systems can be cast as multi-label ranking problems in machine learning.
In such problems, the classifier takes the feature vector of the given data entry and returns the labels in order of estimated confidence value.
In the above examples, the data entries are query keywords (search engine), the document content (keyword-tagging system), or the viewer and landing page information (advertisement system).
In this paper, it is assumed that each data entry is encoded as a real-valued feature vector $\bm{x} \in \mathbb{R}^d$.

However, for such IR systems, the number of candidate labels is generally extremely large (more than several thousand); therefore, it is difficult for classical machine learning approaches to handle these problems.
Previously, \citet{Prabhu_etal2014} advocated \emph{extreme multi-label classification} (XMLC) as a multi-label ranking problem.
Use of ensemble tree models \citep{Prabhu_etal2014} and dimension reduction models \citep{Bhatia_etal2015} has been proposed for XMLC problems.
However, the former cannot achieve high accuracy because of classification errors at internal nodes, whereas the latter is unsuitable for XMLC because of the assumption that the label space can be embedded in the lower-dimension vector space, and because of the instability of the clustering used to handle non-linearities.

Recently, methods without assumptions regarding the properties of the label space have been proposed.
The one-versus-rest approach \citep{Babbar_etal2017} is one such method, which has achieved SOTA performance on many XMLC datasets (see \citet{Bhatia_etal2016}).
The one-versus-rest approach trains each classifier for each label in order to return the corresponding confidence value; therefore, because of the many labels, it is difficult to train the classifiers on XMLC datasets and perform real-time inference using this technique.
To overcome these difficulties, the Parallel Primal-Dual Sparse (PPDSparse, \citet{Yen_etal2017}) has been proposed, which exploits the sparsity in XMLC datasets for efficient training. PPDSparse has improved the SOTA performance.

In this paper, we propose a novel approach named the \emph{Sparse Weighted Nearest-Neighbor Method} for XMLC.
We derive the classifier formula considering fast inference of one-versus-rest classifiers, which exploits various sparsities in XMLC datasets.
Then, we show that the proposed approach can be expressed as a sparse extension of a representer theorem with a linear kernel, and can also be viewed as the vector space model (VSM; \citet[Chapter 7]{Manning_etal2008}) in IR.
Furthermore, there is no need to train our model.
Our method exhibits equivalent performance to SOTA models in real time with a single thread and smaller storage footprint.
In particular, our method improves upon the SOTA performance on a dataset with 3 million labels.

The outline of this paper is as follows.
In Section \ref{sec:previousWorks}, we summarize previous works concerning XMLC problems and note the associated difficulties and problems.
In Section \ref{sec:method}, we explain our method for exploiting several sparsities in XMLC datasets and discuss relationships to other methods.
Furthermore, we confirm the several consistency of our method.
In Section \ref{sec:experiment}, we compare the performance of our method and various SOTA models on XMLC datasets.
Finally, in Section \ref{sec:conclusion}, we conclude this paper and discuss future work.

\section{Previous Works}\label{sec:previousWorks}
In this section, we summarize previous works on XMLC problems, noting the associated difficulties and problems.

\citet{Prabhu_etal2014} have proposed the ensemble tree model known as FastXML; this model is fundamentally random forest, but each tree is trained with the entire dataset and the method for training each splitter function at each node is different.
The splitter function is trained by maximizing the sum of the normalized discounted cumulative gain at the left and right child nodes under the constraint that the predictors at these nodes return the constant probability distribution on the labels.
In FastXML, each probability distribution is empirically estimated from the data points fallen to the node.

FastXML splits the given training dataset recursively; thus, the expected number of splitter functions hit at inference is the logarithm (with base 2) of the size of dataset, which enables fast inference.
However, FastXML yields inferior performance to other modern methods (see \citet{Bhatia_etal2016}).
This is because the error at the internal nodes misdirects the path in the tree significantly, as noted by Babbar et al. (cascading effect; \citet[Subsection 1.2]{Babbar_etal2017}).
Furthermore, when a FastXML model is trained on AmazonCat-13K (see \citet{Bhatia_etal2016}), for example, the model size of each tree is approximately 1 GB; therefore, the storage footprint is several tens of gigabytes in the forest total.

One popular approach involves a dimension reduction method that reduces the (effective) number of labels with co-occurrence.
For example, Sparse Local Embeddings for Extreme Multi-Label Classification (SLEEC, \citet{Bhatia_etal2015}) is a previously developed method combining linear dimension reduction and $k$-means clustering.
Note that simple linear dimension reduction methods (such as Low-Rank Empirical Risk Minimization for Multi-Label Learning (LEML, \citet{Yu_etal2013})) do not perform well on XMLC datasets.
Therefore, SLEEC implements clustering to handle the non-linearity of XMLC datasets.
This method seems to have the capacity to perform faster inference, as reported by \citet{Yen_etal2017}; however, the model size is extremely large and the performance is inferior to other methods.
Furthermore, the training cost is high, as dimension reduction and clustering are involved.

The above methods assume certain properties of the label co-occurrences on XMLC datasets.
That is, FastXML splits the dataset recursively using the label co-occurrences, whereas SLEEC uses dimension reduction exploiting the label co-occurrences.
Thus, even if the correct co-occurrences were found, it would be difficult to improve the prediction accuracy.
This is because the desired classifier takes a feature vector and returns the labels in order, meaning that additional steps such as re-ranking of each label are needed.

Overcoming these difficulties, models without assumptions regarding the XMLC datasets have recently been proposed, which have achieved SOTA performance.
Distributed Sparse Machines for Extreme Multi-Label Classification (DiSMEC, \citet{Babbar_etal2017}) is one such method and employs the one-versus-rest linear classifier,
which is composed of classifiers for each label that return the corresponding label confidence value.
In fact, the one-versus-rest approach is proven to be Bayes optimal \citep{Dembczynski_etal2010}, and other complex approaches such as exploitation of the label co-occurrence (causing inconsistency \citep{Gao_etal2011}) are not required.

DiSMEC has achieved SOTA performance on almost all XMLC datasets; however, it is difficult to train and infer each classifier because of the many labels in these datasets.
Hence, PPDSparse \citep{Yen_etal2017} has been proposed to reduce the training complexity of the one-versus-rest linear classifier by exploiting the sparsities in XMLC datasets.
Although PPDSparse has improved upon the SOTA performance when implemented on one dataset, it is necessary to use all classifiers for each label in inference at present.

In this paper, we consider a fast implementation of the SOTA one-versus-rest linear classification.
Exploiting several sparsities of the XMLC datasets used in our experiments, we propose a more compact and faster method based on nearest-neighbors, i.e., the \emph{Sparse Weighted Nearest-Neighbor Method}, which exhibits equivalent performance to the SOTA models.

\section{Method}\label{sec:method}
In this section, we present the novel \emph{Sparse Weighted Nearest-Neighbor Method} for XMLC.
First, we formulate XMLC problems mathematically.
Then, we derive our model by deriving fast inference with the one-versus-rest linear classifier and exploiting several sparsities of the XMLC datasets used in our experiments (the target datasets).
Furthermore, the several consistencies of our method are analyzed.
Finally, the relationships between our model and other methods are discussed.

\subsection{Mathematical formulation of XMLC problems}
Let $\mathcal{D} := \{(\bm{x}_i, \bm{y}_i) \,|\, \bm{x}_i \in \mathbb{R}^d, \bm{y}_i \in \{-1,+1\}^L, i = 1,\ldots,n\}$ be the given training dataset.
The $i$-th data entry is assumed to be encoded as the feature vector $\bm{x}_i$.
Although $d$ is more than several hundred thousand in the target datasets (Table \ref{tbl:summaryOfDatasets}), the number of non-zero elements (activated features) in each $\bm{x}_i$ is several thousand on average (Table \ref{tbl:featureActivationSummaryOfDatasets}).
Hence, let $\Supp(\bm{x}) := \{j \in \{1,\ldots,d\} \,|\, x_j \neq 0\}$ be the support of $\bm{x} \in \mathbb{R}^d$; that is, $\Supp(\bm{x})$ is the set of the non-zero element identifiers (IDs) of $\bm{x}$.
Further, $\bm{y}_i$ is the label vector for the $i$-th data entry, of which the $l$-th element $y_{il}$ is $+1$ if the $i$-th data entry has the label $l$, and $-1$ otherwise.
In addition, let $\bm{y}_+ := ((y_l)_+)_{l=1}^L := (\max\{0, y_l\})_{l=1}^L \in \{0, +1\}^L$.

A model $\mathcal{F}$ is a subset of the function (classifier) $\bm{f}: \mathbb{R}^d \to \mathbb{R}^L$, where $\bm{f}$ takes the $\bm{x}$ encoding the given data entry and returns the label confidence vector $\bm{z} := \bm{f}(\bm{x})$.
If the classifier predicts that the data entry has the label $l$, then $z_l$ is positive; otherwise, $z_l$ is negative.
The absolute value of $z_l$ indicates the prediction strength.
Using the classifier, we can rank labels in order of confidence.

Then, a classifier $\hat{\bm{f}} \in \mathcal{F}$ is trained by minimizing the empirical loss with the loss function $\mathcal{L}: \{-1, +1\}^L \times \mathbb{R}^L \to \mathbb{R}$; that is,
\begin{align*}
\hat{\bm{f}} := \argmin_{\bm{f} \in \mathcal{F}} \sum_{i=1}^n \mathcal{L}(\bm{y}_i, \bm{f}(\bm{x}_i)).
\end{align*}
As discussed in Section \ref{sec:previousWorks}, we use a one-versus-rest loss function:
\begin{align*}
\mathcal{L}(\bm{y}, \bm{z}) := \sum_{l=1}^L L(y_l, z_l).
\end{align*}
Here, $L: \{-1, +1\} \times \mathbb{R} \to \mathbb{R}$ is a loss function for binary-class classification.

\subsection{Fast Inference with One-Versus-Rest Linear Classifier}
In this subsection, we consider fast inference with the trained one-versus-rest linear classifier $\bm{f}$, which predicts the label confidence vector as follows:
\begin{align*}
\bm{f}(\bm{x}) &= W\bm{x} = (\bm{w}_l^\top\bm{x})_{l=1}^L = \left(\sum_{j=1}^d w_{lj}x_j\right)_{l=1}^L, \\
  &\qquad W := (\bm{w}_l^\top)_{l=1}^L: L \times d.
\end{align*}
Exploiting the empirical fact that the feature vectors in the target XMLC datasets are sparse (Table \ref{tbl:featureOccurrenceSummaryOfDatasets}), we can rewrite the above expression as follows:
\begin{align*}
\bm{f}(\bm{x}) = \left(\sum_{j: x_j \neq 0} w_{lj}x_j\right)_{l=1}^L.
\end{align*}
Hence, we can perform more rapid calculations, as only non-zero features of the given feature vector are considered.
This approach can be implemented efficiently with the index list mapping from the feature ID to the index of the target label ID and the corresponding weight.

\subsection{Sparse Weighted Nearest-Neighbor Method}
If there are a large number of labels in a dataset, the number of labels exceeds the number of data entries.
Thus, designing the classifier using the data entries themselves is more efficient than training classifiers for each label.
Some of the target XMLC datasets have many labels, which occur once only (Table \ref{tbl:summaryOfDatasets}); therefore, we can use the data entries themselves instead of the corresponding classifiers.

Based on this approach, we propose the \emph{Sparse Weighted Nearest-Neighbor Method} for XMLC.
The classifier is defined as follows:
\begin{align*}
\bm{f}(\bm{x}) := \sum_{r=1}^S \max\{\mathrm{Sim}(\bm{x}, \bm{x}_{N(r; \bm{x})}), 0\}^\alpha (\bm{y}_{N(r; \bm{x})})_+.
\end{align*}
Here, $S$ is the hyper-parameter (hyper-param) specifying the size of the neighborhood and $\alpha \geq 0$ is the smoothing parameter for weighted votes.
Further, $N(1; \bm{x}),\ldots,N(S; \bm{x})$ are the indices of the data entries in the training dataset $\mathcal{D}$, which belong to the neighborhood in decreasing order of similarity $\mathrm{Sim}$, defined as follows:
\begin{align*}
\mathrm{Sim}(\bm{x}, \bm{x}_i) := J(\bm{x}, \bm{x}_i)^\beta \frac{\bm{x}^\top\bm{x}_i}{\|\bm{x}\|_2\|\bm{x}_i\|_2}.
\end{align*}
Here, $J(\bm{x}, \bm{x}_i)$ is the Jaccard similarity between $\bm{x}$ and $\bm{x}_i$, which is defined as follows:
\begin{align*}
J(\bm{x}, \bm{x}_i) := \frac{|\Supp(\bm{x}) \cap \Supp(\bm{x}_i)|}{|\Supp(\bm{x}) \cup \Supp(\bm{x}_i)|} \in [0, 1].
\end{align*}
The Jaccard similarity value increases with the number of common non-zero elements of $\bm{x}$ and $\bm{x}_i$.
Here, $\beta \geq 0$ is the hyper-param for balancing the Jaccard similarity and cosine similarity.

Note that we adopt the cosine similarity in this technique because our preliminary experiments confirmed that use of the $L_1$- or $L_2$-distance degrades performance.
Considering the $L_2$-distance among $\bm{x}$, $\bm{x}_i$, and $\bm{x}_j$ normalized to unit length, we find
\begin{align*}
\left\|\frac{\bm{x}}{\|\bm{x}\|_2} - \frac{\bm{x}_i}{\|\bm{x}_i\|_2}\right\|_2
  \lessgtr \left\|\frac{\bm{x}}{\|\bm{x}\|_2} - \frac{\bm{x}_j}{\|\bm{x}_j\|_2}\right\|_2
~\Leftrightarrow~
  \frac{\bm{x}^\top\bm{x}_j}{\|\bm{x}\|_2\|\bm{x}_j\|_2}
  \lessgtr \frac{\bm{x}^\top\bm{x}_i}{\|\bm{x}\|_2\|\bm{x}_i\|_2}.
\end{align*}
Hence, the nearer feature vector in terms of $L_2$-distance has the larger value in terms of the cosine similarity.

Balancing of the Jaccard similarity and cosine similarity is adopted in this method because, again, we found empirically that the vote given by the best-matched training data entry is sometimes ignored by the votes from the nearer training data entries.
As an example, assume that $\mathcal{D}$ has one data entry for which the feature vector support is $\{1, 2, 4\}$ and the labels are 1 and 2, and four data entries for which the feature vector support is $\{1, 2, 4, 5, 8\}$ and the labels are 3, 5, and 6.
If the data point for which the feature vector support is $\{1, 2, 4\}$ is given, the classifier returns labels 3, 5, and 6 before 1 and 2 based on the cosine similarities, even if the given data point has labels 1 and 2.
In this case, we let the classifier discount the similarity to the latter data entries, because the best matched data entry have identical information in the feature vector.
Therefore, we designed the classifier using the Jaccard similarity to measure the matching of the supports between two feature vectors if needed.

If $\beta$ is non-negative integer, then the similarity $\mathrm{Sim}$ satisfies the properties of a kernel function:
\begin{lemma}[Sim is Kernel Function]\label{lemma:simIsKernel}
If $\beta$ is non-negative integer, then $\mathrm{Sim}$ satisfies the following properties:
  (i: Symmetry) For any $\bm{x}, \bm{x}' \in \mathbb{R}^d$, $\mathrm{Sim}(\bm{x}, \bm{x}') = \mathrm{Sim}(\bm{x}', \bm{x})$, and
  (ii: Positive Definiteness) For any $\bm{x}_1,\ldots,\bm{x}_n \in \mathbb{R}^d$ and $\alpha_1,\ldots,\alpha_n \in \mathbb{R}$, $\sum_{i=1}^n\sum_{j=1}^n \alpha_i\alpha_j \mathrm{Sim}(\bm{x}_i, \bm{x}_j) \geq 0$.
\end{lemma}

\begin{proof}
(i) From the definition of $\mathrm{Sim}$, it is symmetric.

(ii) The inner product on $\mathbb{R}^d$ is positive definite, so by \citet[Lemma 5.2]{Mohri_etal2012}, cosine similarity, that is the normalized version of the inner product, is also.
The Jaccard similarity is positive definite \citep[Theorem 2.1]{Bouchard_etal2013}.
By \citet[Theorem 5.3]{Mohri_etal2012}, the product of two positive definite functions is also positive definite; thereby the integer power of a positive definite function is also.
Hence, $\mathrm{Sim}$ is positive definite.
\end{proof}

Hence, $\mathrm{Sim}$ induces a reproducing kernel Hilbert space, therefore, $\mathrm{Sim}$ can be viewed as a natural similarity measure.

\subsection{Sparsities Exploited by Sparse Weighted Nearest-Neighbor Method}
In addition to a fast inference method exploiting the sparsity of each feature vector in the target XMLC datasets, we can exploit two other sparsities of the target XMLC datasets, thereby improving the speed and efficiency of our method.

First, each data entry has several tens of labels on average (Table \ref{tbl:labelPerEntrySummaryOfDatasets}); therefore, we can reduce the number of candidate labels with the location of the given data entry.
Furthermore, it would be difficult to ensure that the unattached label is truly irrelevant if the data entry does not have that label.
Thus, it would be better to use only the labels attached to each training data entry, so we use only positive parts of the label vector $(\bm{y}_i)_+$ instead of $\bm{y}_i$ itself.

Second, each feature occurs only several hundred times on average (Table \ref{tbl:featureOccurrenceSummaryOfDatasets}); therefore, if we prepare the index list from the feature ID to the index of the data entry ID and the feature value, we can reduce the target training data entries significantly.
In the target XMLC datasets, each element of a feature vector is a monotone increasing function of the count of each word or event in the data entry; therefore, it would be plausible to ignore the votes of training data entries with different supports.
Therefore, the weighted vote based on training data entries with positive $\mathrm{Sim}$ only is selected for more efficient and accurate inference.

\subsection{Several Consistencies}
In this subsection, we confirm several consistencies of our method.
Their proofs are in Appendix \ref{sec:appendixProofs}.

We prepare the following notations:
The given training dataset is composed of the independent and identical random variables $(\bm{X}_i, \bm{Y}_i) \in \mathcal{X} \times \{-1, +1\}^L, i = 1,\ldots,n$ drawn from the unknown population distribution.
Since our classifier does not depend on the length of the feature vectors, let $\mathcal{X}$ be the unit sphere $B_1(\bm{0}) := \{ \bm{x} \in \mathbb{R}^d \,|\, \|\bm{x}\|_2 = 1 \}$.
Hence, we assume that each feature vector is normalized to unit length.

We assume that the population distribution satisfies the following conditions.
First, it has a probability density function $f$.
If the feature vectors are discrete, the following discussion can be applied to the probability function instead of the probability density function.
Second, the support is entire $\mathcal{X}$, that is, for any $\bm{x} \in \mathcal{X}$ and $\epsilon > 0$,
\begin{align*}
\Prob{\bm{X} \in B_\epsilon(\bm{x})} > 0.
\end{align*}

Furthermore, we assume that $\alpha > 0$, because if $\alpha$ is 0, our method becomes an unweighted nearest-neighbor method which has already been analyzed well \citep{Biau_etal2015}.

First, we show that the similarity in the $S$-nearest-neighbors converges to $1$ in probability:
\begin{lemma}[Similarity Consistency]\label{lemma:similarityConsistency}
Fix any $S = o(n) \in \mathbb{Z}_{\geq 1}$.
For any $\bm{x} \in \mathcal{X}$, $r \in \{1,\ldots,S\}$,
\begin{align*}
\mathrm{Sim}(\bm{x}, \bm{X}_{\pi(r)}) \to 1 \text{ in probability},
\end{align*}
as $n$ goes to infinity.
Here, $\bm{X}_{\pi(1)},\ldots,\bm{X}_{\pi(n)}$ are $\bm{X}_1,\ldots,\bm{X}_n$ in decreasing order of similarity $\mathrm{Sim}$.
\end{lemma}

Hence, we obtain the following weight consistency:
\begin{corollary}[Weight Consistency]\label{cor:weightConsistency}
Under the assumption of Lemma \ref{lemma:similarityConsistency},
\begin{align*}
W_{\pi(r)}(\bm{x}) := \max\{\mathrm{Sim}(\bm{x}, \bm{X}_{\pi(r)}), 0\}^\alpha \to 1 \text{ in probability}.
\end{align*}
\end{corollary}

Next, it follows that any data entry in the nearest-neighbors of a given data entry $\bm{x}$ converges to $\bm{x}$ in probability:
\begin{lemma}[Data Point Consistency]\label{lemma:dataPointConsistency}
Under the assumption of Lemma \ref{lemma:similarityConsistency},
\begin{align*}
\bm{X}_{\pi(r)} \to \bm{x} \text{ in probability},
\end{align*}
as $n$ goes to infinity, that is,
\begin{align*}
\|\bm{X}_{\pi(r)} - \bm{x}\|_2^2 \to 0 \text{ in probability}.
\end{align*}

Furthermore,
\begin{align*}
\E{\|\bm{X}_{\pi(r)} - \bm{x}\|_2^2} \to 0.
\end{align*}
\end{lemma}

At last, we prove the risk consistency of our method for estimating each label probability $\CondProb{Y_l = 1}{\bm{X} = \bm{x}}$.
Hence, we fix any label ID $l$, and we rewrite the classifier of our method for estimating the label probability as follows:
\begin{align*}
\hat{f}_n(\bm{x}) := \hat{f}_{n,l}(\bm{x}) := \sum_{i=1}^n V_i(\bm{x})(y_{il})_+.
\end{align*}
Here,
\begin{align*}
V_i(\bm{x}) := \frac{W_i(\bm{x})}{\sum_{j=1}^n W_j(\bm{x})}.
\end{align*}
This is the estimator of the probability for a given $l$, thus let $f(\bm{x}) := f_l(\bm{x}) := \CondProb{Y_l = 1}{\bm{X} = \bm{x}}$.
For notation convenience, we omit the label ID $l$, and use the positive version, so we define $y_i := (y_{il})_+$.

The following lemma gives the upper bound of the squared loss for estimating the label probability:
\begin{lemma}[Upper Bound of Squared Loss]\label{lemma:squaredLossUpperBound}
Assume that there exists $\sigma^2 \in [0, \infty)$ such that, for any $\bm{x} \in \mathcal{X}$,
\begin{align*}
\CondVar{Y}{\bm{X} = \bm{x}} \leq \sigma^2,
\end{align*}
and $f$ is Lipschitz continuous, that is, there exists $L > 0$ such that $|f(\bm{x}) - f(\bm{x}')| \leq L\|\bm{x} - \bm{x}'\|_2, \forall \bm{x}, \bm{x}' \in \mathbb{R}^d$.
Then, the squared loss of $\hat{f}_n(\bm{X})$ is bounded as follows:
\begin{align*}
\E{(\hat{f}_n(\bm{X}) - f(\bm{X}))^2}
  \leq \sigma^2\sum_{r=1}^S \E{V_{\pi(r)}(\bm{X})^2} + L^2\sum_{r=1}^S \E{V_{\pi(r)}(\bm{X})\|\bm{X}_{\pi(r)} - \bm{X}\|_2^2}.
\end{align*}
\end{lemma}
Here, the first and second terms correspond to the variance and bias terms, respectively (see the proof).
The proof is based on the proof of \citet[Theorem 4]{Biau_etal2010}.
Notice that $V_i(\bm{X})$ is not deterministic, and depends on both $\bm{X}$ and $\bm{X}_1,\ldots,\bm{X}_n$.

Furthermore, the variance term of our method is asymptotically bounded by $\sigma^2/S$ as $n$ goes to infinity for $S = o(n)$.
This is because, by Corollary \ref{cor:weightConsistency}, $V_{\pi(r)}(\bm{X}) \to 1/S$ in probability and the uniform integrability.
This means that the variance of our method is not worse than an unweighted nearest-neighbor method asymptotically.
Next, the bias term converges to $0$, because $V_{\pi(r)}(\bm{X}) \to 1/S$ in probability and Lemma \ref{lemma:dataPointConsistency}.
Hence, we obtain the following risk consistency:
\begin{proposition}[Risk Consistency]\label{prop:riskConsistency}
Under the assumption of Lemma \ref{lemma:squaredLossUpperBound}, choosing $S = o(n)$,
\begin{align*}
\E{(\hat{f}_n(\bm{X}) - f(\bm{X}))^2} \to 0,
\end{align*}
as $n$ and $S$ go to infinity.
\end{proposition}

By the risk consistency, we can prove the Bayes consistency as follows (\citet[Theorem 17.1]{Biau_etal2015} and Cauchy-Schwarz inequality):
\begin{align*}
&\left|\Prob{\sign\left(\hat{f}_n(\bm{X}) - \frac{1}{2}\right) \neq Y} - \Prob{\sign\left(f(\bm{X}) - \frac{1}{2}\right) \neq Y}\right| \\
&\leq 2\sqrt{\E{(\hat{f}_n(\bm{X}) - f(\bm{X}))^2}} \to 0.
\end{align*}

We have not yet derived the rate of convergence of our methods, because no quantity has been found like the metric covering radius satisfying the relation $\mathcal{N}^{-1}(r; \Supp(\mu)) \leq \mathcal{N}^{-1}(1; \Supp(\mu))r^{-1/d}$ \citep{Biau_etal2010}.
We leave this analysis for future works.

\subsection{Relationships to Other Methods}
First, in statistics and machine learning, the $L_2$-regularized linear classification can be solved as follows (representer theorem, \citet[Subsection 5.3.2]{Mohri_etal2012}):
\begin{align*}
\hat{\bm{f}}(\bm{x}) &= \sum_{i=1}^n \bm{\alpha}_i \circ (\bm{x}_i^\top\bm{x}) \bm{y}_i, \\
  &\qquad \bm{\alpha}_i := \bm{\alpha}_i(\bm{x}_1,\ldots,\bm{x}_n) \in \mathbb{R}^L, i = 1,\ldots,n.
\end{align*}
Here, $\bm{\alpha} \circ \bm{z} := (\alpha_l \cdot z_l)_{l=1}^L$.
Our classifier can be viewed as a sparse generalization of this formula, where $\bm{\alpha}_i$ is non-negative and depends on not only $\bm{x}_1,\ldots,\bm{x}_n$, but also $\bm{x}$.
Here, note that the denominator of the cosine similarity can be absorbed by $\bm{\alpha}_i$, and the model normalizes $\bm{x}$ to unit length at first.

We have not yet determined the mathematical relationship of our classifier formula to the classical representation theorem.
However, Yen et al. have proposed a method for which $(\bm{\alpha}_i)_{i=1}^n$ would be sparse, with no dependence on $\bm{x}$ \citep[Eq. (10)]{Yen_etal2017}.
This mathematical problem is left for future work.

Second, our method can be implemented efficiently with feature index list mapping from the feature ID to the index of the data entry ID and the feature value.
This is reminiscent of classical search engines maintaining the index list for efficient document retrieval \citep[Chapter 4]{Manning_etal2008}.
In our model, this corresponds to the sparsity where the unattached labels of each training data entry are ignored.
In addition, our method uses the similarity extending cosine similarity between feature vectors, which is used in the VSM \citep[Chapter 7]{Manning_etal2008}.
In our model, this corresponds to the sparsity where the training data entries for which the similarity is negative are ignored.

Finally, recall that FastXML is one of the fast methods targeting XMLC problems, which consists of a random forest method using the entire training dataset with a linear function as each splitter.
In general, such a random forest technique is known to be related to the (adaptive) nearest-neighbor method \citep{Lin_etal2002}.
This is because the random forest method predicts the result based on the average of the split region on which the given data entry falls; thus, we can expect that each split region reflects the neighbor structure learned by the raw data.
In our method, the similarity is designed manually to define the neighborhood for fast inference; therefore, future work will be conducted to combine our method with the FastXML-style adaptive nearest-neighbor method.

\section{Experiment}\label{sec:experiment}
In this section, we report experiments in which our method was compared with SOTA models on XMLC datasets.

\subsection{Settings}
We used six XMLC datasets available from the Extreme Classification Repository \citep{Bhatia_etal2016}.
Table \ref{tbl:summaryOfDatasets} presents a summary of the six datasets employed in this work, which are described in detail in the next subsection.
The suffix of each dataset name is the number of labels (e.g., AmazonCat-13K has approximately 13 thousand labels).
Four of the datasets, excluding AmazonCat-13K and WikiLSHTC-325K, have more labels than the number of data entries in the training dataset.

\begin{table*}[t]
  \caption{Six datasets used in experiment (Dim.: dimension)}\label{tbl:summaryOfDatasets}
  \begin{center}
  \begin{tabular}{l|rrr}
    \hline
    Name           & Feature Vector Dim. & Training Set Size & Test Set Size \\
    \hline\hline
    AmazonCat-13K  & 203,882              & 1,186,239            & 306,782          \\
    Wiki10-31K     & 101,938              & 14,146               & 6,616            \\
    Delicious-200K & 782,585              & 196,606              & 100,095          \\
    WikiLSHTC-325K & 1,617,899            & 1,778,351            & 587,084          \\
    Amazon-670K    & 135,909              & 490,449              & 153,025          \\
    Amazon-3M      & 337,067              & 1,717,899            & 742,507          \\
    \hline
  \end{tabular}
  \end{center}
\end{table*}

In the XMLC dataset, because of the number of candidate labels, it is difficult to confirm that the labels unattached to each data entry are inappropriate for the data entry.
Hence, we used \emph{Precision}@$K$ as the criterion for measuring performance.

Let $\mathcal{D}' := \{(\bm{x}'_i, \bm{y}'_i) \,|\, \bm{x}'_i \in \mathbb{R}^d, \bm{y}'_i \in \{-1,+1\}^L, i = 1,\ldots,n'\}$ be the test dataset.
Then, the Precision@$K$ of the classifier $\bm{f}(\bm{x})$ is defined as follows:
\begin{align*}
\mathrm{Precision@}K(\mathcal{D}') := \frac{1}{n'}\sum_{i=1}^{n'} \frac{1}{K}\sum_{k=1}^K (y'_{i,\rank(\bm{f}(\bm{x}'_i), k)})_+.
\end{align*}
Here, $\rank(\bm{z}, k)$ is the $k$-th largest element ID of $\bm{z}$.
If $\bm{z}$ is the label confidence vector, $\rank(\bm{z}, k)$ is the $k$-th predicted label.
Furthermore, if a data entry with less than $K$ labels exists, it is impossible for any model to achieve a result of 100\%.
Thus, we report the maximum Precision@$K$ value for each method examined in this section.
 
To measure the inference time of our method, we used a PC with Intel Core i7-7700K @ 4.20GHz (4-core, 8-thread), 64-GB memory running Windows 10 Home.
We have published our implementation \emph{sticker} \citep{Aoshima_2018} written in Golang \citep{Golang_2009}.

\subsection{Dataset Features}
We briefly summarize the following four features of the six datasets employed in this work.

First, Table \ref{tbl:labelOccurrenceSummaryOfDatasets} lists the five-number summary and average number of label occurrences in the six training datasets.
Most labels occur only several times; therefore, it is appropriate to design models with these tail labels to achieve higher accuracy.

\begin{table*}[t]
  \caption{Five-number summary and average number of label occurrences in six training datasets}\label{tbl:labelOccurrenceSummaryOfDatasets}
  \begin{center}
  \begin{tabular}{l||rrrrr|r}
    \hline
    Name           & Minimum & 1st Qu. & Median & 3rd Qu. & Maximum & Average \\
    \hline\hline
    AmazonCat-13K  & 1       & 7       & 28     & 111     & 335,211 & 449     \\
    Wiki10-31K     & 1       & 2       & 2      & 4       & 11,411  & 9       \\
    Delicious-200K & 1       & 3       & 6      & 17      & 64,548  & 74      \\
    WikiLSHTC-325K & 1       & 2       & 5      & 13      & 293,936 & 18      \\
    Amazon-670K    & 1       & 2       & 3      & 4       & 1,826   & 4       \\
    Amazon-3M      & 1       & 3       & 7      & 19      & 12,014  & 22      \\
    \hline
  \end{tabular}
  \end{center}
\end{table*}

Second, Table \ref{tbl:labelPerEntrySummaryOfDatasets} lists the five-number summary and average number of labels attached to each data entry in the six training datasets.
Even though the dimension of the feature vectors is more than several hundred thousand, each data entry has only several tens of labels; therefore, we can expect that the number of candidate labels for a given data entry can be reduced.

\begin{table*}[t]
  \caption{Five-number summary and average number of labels attached to each data entry in six training datasets}\label{tbl:labelPerEntrySummaryOfDatasets}
  \begin{center}
  \begin{tabular}{l||rrrrr|r}
    \hline
    Name           & Minimum & 1st Qu. & Median & 3rd Qu. & Maximum & Average \\
    \hline\hline
    AmazonCat-13K  & 1       & 3       & 4      & 6       & 57      & 5       \\
    Wiki10-31K     & 1       & 13      & 19     & 25      & 30      & 19      \\
    Delicious-200K & 1       & 9       & 27     & 77      & 13,203  & 76      \\
    WikiLSHTC-325K & 1       & 1       & 2      & 4       & 198     & 3       \\
    Amazon-670K    & 1       & 5       & 6      & 7       & 7       & 5       \\
    Amazon-3M      & 1       & 7       & 20     & 65      & 100     & 36      \\
    \hline
  \end{tabular}
  \end{center}
\end{table*}

Third, Table \ref{tbl:featureActivationSummaryOfDatasets} lists the five-number summary and average number of feature activations on each data entry in the six training datasets.
Each data entry has only several hundred activated features; therefore, the feature vectors are very sparse, and we can efficiently calculate the cosine similarities between vectors with common activated features only.

\begin{table*}[t]
  \caption{Five-number summary and average number of feature activations on each data entry in six training datasets}\label{tbl:featureActivationSummaryOfDatasets}
  \begin{center}
  \begin{tabular}{l||rrrrr|r}
    \hline
    Name           & Minimum & 1st Qu. & Median & 3rd Qu. & Maximum & Average \\
    \hline\hline
    AmazonCat-13K  & 1       & 22      & 45     & 91      & 2,113   & 71      \\
    Wiki10-31K     & 8       & 282     & 520    & 925     & 3,288   & 673     \\
    Delicious-200K & 1       & 35      & 115    & 340     & 82,820  & 301     \\
    WikiLSHTC-325K & 1       & 16      & 30     & 54      & 5,053   & 42      \\
    Amazon-670K    & 1       & 23      & 48     & 96      & 2,025   & 76      \\
    Amazon-3M      & 0       & 20      & 37     & 65      & 2,443   & 49      \\
    \hline
  \end{tabular}
  \end{center}
\end{table*}

Finally, Table \ref{tbl:featureOccurrenceSummaryOfDatasets} lists the five-number summary and average number of feature occurrences in the six training datasets.
Most features occur only several hundred times; therefore, we can expect that the length of each feature index is considerably shorter than the size of the training dataset.

\begin{table*}[t]
  \caption{Five-number summary and average number of feature occurrences in six training datasets}\label{tbl:featureOccurrenceSummaryOfDatasets}
  \begin{center}
  \begin{tabular}{l||rrrrr|r}
    \hline
    Name           & Minimum & 1st Qu. & Median & 3rd Qu. & Maximum & Average \\
    \hline\hline
    AmazonCat-13K  & 1       & 6       & 13     & 53      & 499,293 & 414     \\
    Wiki10-31K     & 1       & 5       & 11     & 36      & 7,103   & 94      \\
    Delicious-200K & 1       & 4       & 6      & 14      & 82,028  & 76      \\
    WikiLSHTC-325K & 1       & 1       & 1      & 3       & 274,175 & 54      \\
    Amazon-670K    & 1       & 6       & 13     & 54      & 209,059 & 273     \\
    Amazon-3M      & 1       & 5       & 8      & 21      & 330,125 & 251     \\
    \hline
  \end{tabular}
  \end{center}
\end{table*}

\subsection{Results}
Table \ref{tbl:result} lists the results for Precision@\{1,3,5\} (P@\{1,3,5\}) on the six datasets, comparing our model and the SOTA models of each dataset.
The SOTA performance was quoted from \citet{Bhatia_etal2016}, because we did not have sufficient machines to verify those complex methods.
Bold values indicate that our model performed comparable or better than the SOTA model by approximately 1\%.
Note that Propensity scored FastXML (PfastreXML, \citet{Jain_etal2016}) is a method constituting a small improvement on FastXML.

\begin{table*}[ht]
  \caption{Result for Six Datasets}\label{tbl:result}
  \begin{center}
  \begin{tabular}{l|l||r|r|r}
    \hline
    Name           &              & Our Model                   & SOTA         & Maximum \\
    \hline\hline
    AmazonCat-13K  & P@1          & 91.07\%                     & 93.40\%      & 100.0\% \\
                   & P@3          & 76.99\%                     & 79.10\%      & 91.74\% \\
                   & P@5          & 62.59\%                     & 64.10\%      & 77.24\% \\
                   \cline{2-5}
                   & hyper-params & $S=25, \alpha=1.0, \beta=1$ & (DiSMEC)     &         \\
    \hline
    Wiki10-31K     & P@1          & \bf 84.89\%                 & \bf 85.20\%  & 100.0\% \\
                   & P@3          & \bf 74.65\%                 & \bf 74.60\%  & 99.99\% \\
                   & P@5          & \bf 64.88\%                 & \bf 65.90\%  & 99.93\% \\
                   \cline{2-5}
                   & hyper-params & $S=20, \alpha=1.0, \beta=1$ & (DiSMEC)     &         \\
    \hline
    Delicious-200K & P@1          & \bf 48.05\%                 & \bf 47.85\%  & 100.0\% \\
                   & P@3          & \bf 42.00\%                 & \bf 42.21\%  & 98.10\% \\
                   & P@5          & \bf 39.02\%                 & \bf 39.43\%  & 96.08\% \\
                   \cline{2-5}
                   & hyper-params & $S=75, \alpha=1.0, \beta=0$ & (SLEEC)      &         \\
    \hline
    WikiLSHTC-325K & P@1          & 58.29\%                     & 64.40\%      & 100.0\% \\
                   & P@3          & 37.27\%                     & 42.50\%      & 76.18\% \\
                   & P@5          & 27.53\%                     & 31.50\%      & 57.40\% \\
                   \cline{2-5}
                   & hyper-params & $S=25, \alpha=1.0, \beta=1$ & (DiSMEC)     &         \\
    \hline
    Amazon-670K    & P@1          & 41.43\%                     & 45.32\%      & 100.0\% \\
                   & P@3          & 36.79\%                     & 40.37\%      & 93.40\% \\
                   & P@5          & 33.70\%                     & 36.92\%      & 87.65\% \\
                   \cline{2-5}
                   & hyper-params & $S=75, \alpha=2.0, \beta=1$ & (PPDSparse)  &         \\
    \hline
    Amazon-3M      & P@1          & \bf 45.69\%                 & \bf 43.83\%  & 100.0\% \\
                   & P@3          & \bf 42.53\%                 & \bf 41.81\%  & 95.15\% \\
                   & P@5          & \bf 40.39\%                 & \bf 40.09\%  & 91.89\% \\
                   \cline{2-5}
                   & hyper-params & $S=75, \alpha=2.0, \beta=1$ & (PfastreXML) &         \\
    \hline
  \end{tabular}
  \end{center}
\end{table*}

The hyper-params were selected as follows.
First, we tried $(\alpha, \beta) = (0.0, 0)$, $(0.5, 0)$, $(1.0, 0)$, $(1.0, 1)$, $(2.0, 0)$, $(2.0, 1)$ with $S$ set as the average number of labels attached to each data entry (see Table \ref{tbl:labelPerEntrySummaryOfDatasets}).
Then, we selected the best $(\alpha, \beta)$ and tried $S = 25, 50, 75$ as well, for later discussion.

Our model can run on 4-GB memory and the model size is several hundred MB to 1.5 GB at most.
Training is not needed (preparation for constructing the feature index list is performed in 1 min including input/output (I/O) time).
Furthermore, our model processes each data entry on a single thread in 24.1 (AmazonCat-13K), 1.47 (Wiki10-31K), 6.02 (Delicious-200K), 20.4 (WikiLSHTC-325K), 5.74 (Amazon-670K), and 19.7 ms (Amazon-3M) on average.

Here, we discuss the performance comparison of our method and the SOTA models.
We confirmed in our preliminary experiment that our one-versus-rest linear classifiers targeting only the top-500 frequently occurring labels process each data entry in approximately 0.1 ms; thus, full one-versus-rest processing is completed in approximately 20 ms, at minimum.
As reported in \citet[Table 1]{Yen_etal2017}, DiSMEC and PPDSparse require several hundred milliseconds of processing time, while PfastreXML and SLEEC require several milliseconds; thus, our method exhibits competitive speed compared to these SOTA models.
Note that the other models have sizes of at least several GB; therefore, it would be difficult to infer their performance in real time on 4-GB memory, which can be used to run our method.

Our model exhibited equivalent performance to the SOTA models on Wiki10-31K and Delicious-200K, and improved upon the SOTA performance for Amazon-3M.
Our model seems to work well under the condition that the number of labels is larger than the training dataset size; therefore, our model is unlikely to perform well on AmazonCat-13K and WikiLSHTC-325K. However, the performance for those models was acceptable.
We expected our model to perform well on Amazon-670K, but its performance was slightly inferior to that of the PPDSparse SOTA model.
However, we noted that our method performs better on WikiLSHTC-325K and Amazon-670K than other models, except for DiSMEC and PPDSparse, which encounter difficulties when applied to a larger dataset.

Apart from Delicious-200K, the case of $\beta = 1$ exhibited superior performance, which indicates that the Jaccard similarity fits well.
In our preliminary experiments, we found that the case of $S = 25, 50, 75$ with the best $\alpha$ and $\beta$ on each dataset did not degrade to a large extent (within several percent), excluding Delicious-200K ($\beta = 0$ was best).
In general, the nearest-neighbor method degrades performance by increasing $S$; however, for our method, the degradation is not very large.
We consider that this is because the Jaccard similarity is smaller on the training data entries, which should be ignored.

\section{Conclusion}\label{sec:conclusion}
In this paper, we proposed a novel method termed the \emph{Sparse Weighted Nearest-Neighbor Method}.
Our method can be derived through a fast inference with an SOTA one-versus-rest linear classifier (similar to DiSMEC \citep{Babbar_etal2017} and PPDSparse \citep{Yen_etal2017}) applied to XMLC problems.
Furthermore, we discussed the relationship of our classifier formula to other methods; it is reminiscent of the vector space model and the representer theorem for linear classification.
In addition, we conducted experiments on six popular XMLC benchmark datasets and confirmed that our method exhibits equivalent performance to the SOTA models in real time with a single thread and smaller footprint.
In particular, our method improved upon the SOTA performance on a dataset with 3 million labels.

Our nearest-neighbor-based method is most suited to datasets in which the number of labels is smaller than the number of data entries.
Hence, there is a need for a dataset compression mechanism in order to reduce the noise without generating loss of label information and model size.
This mechanism would directly accelerate the inference and will be investigated in future work.


\appendix
\section{Proofs}\label{sec:appendixProofs}
\begin{proof}[Proof of Lemma \ref{lemma:similarityConsistency}]
For any $\epsilon > 0$, it is sufficient to show that
\begin{align*}
\lim_{n \to \infty} \Prob{|\mathrm{Sim}(\bm{x}, \bm{X}_{\pi(r)}) - 1| \geq \epsilon} = 0.
\end{align*}
Since $\mathrm{Sim}(\bm{X}_{\pi(r)}, \bm{x}) \in [0, 1]$, it is sufficient to show that
\begin{align*}
\lim_{n \to \infty} \Prob{\mathrm{Sim}(\bm{x}, \bm{X}_{\pi(r)}) \leq 1 - \epsilon} = 0.
\end{align*}
Here, $\bm{X}_{\pi(1)},\ldots,\bm{X}_{\pi(n)}$ are sorted in decreasing order of similarity $\mathrm{Sim}$, thus,
\begin{align*}
0 &\leq \Prob{\mathrm{Sim}(\bm{x}, \bm{X}_{\pi(r)}) \leq 1 - \epsilon} \\
&= \Prob{\mathrm{Sim}(\bm{x}, \bm{X}_1) > 1 - \epsilon}^{-1}
  \cdot \sum_{s=0}^{r-1} \binom{n}{s}
    \Prob{\mathrm{Sim}(\bm{x}, \bm{X}_1) > 1 - \epsilon}^{s+1}
    \Prob{\mathrm{Sim}(\bm{x}, \bm{X}_1) \leq 1 - \epsilon}^{n-s}.
\end{align*}
By \citet[Lemma 16]{Biau_etal2010},
\begin{align*}
0 &\leq \Prob{\mathrm{Sim}(\bm{x}, \bm{X}_{\pi(r)}) \leq 1 - \epsilon} \\
&\leq \Prob{\mathrm{Sim}(\bm{x}, \bm{X}_1) \geq 1 - \epsilon}^{-1}\frac{r}{n+1} \\
&\leq \Prob{\mathrm{Sim}(\bm{x}, \bm{X}_1) \geq 1 - \epsilon}^{-1}\frac{o(n)}{n+1}
  \to 0.
\end{align*}
Hence, we have proven the statement.
\end{proof}

\begin{proof}[Proof of Lemma \ref{lemma:dataPointConsistency}]
For any $\epsilon > 0$,
\begin{align*}
1 &\geq \Prob{\|\bm{X}_{\pi(r)} - \bm{x}\|_2 \leq \epsilon} \\
&\geq \Prob{J(\bm{X}_{\pi(r)}, \bm{x})^\beta(\bm{X}_{\pi(r)}^\top\bm{x}) \geq 1 - \frac{\epsilon}{2}} \\
&= \Prob{\mathrm{Sim}(\bm{x}, \bm{X}_{\pi(r)}) \geq 1 - \frac{\epsilon}{2}}.
\end{align*}
Since $\mathrm{Sim}(\bm{x}, \bm{X}_{\pi(r)})$ takes a value in $[0, 1]$, by Lemma \ref{lemma:similarityConsistency},
\begin{align*}
\Prob{\mathrm{Sim}(\bm{x}, \bm{X}_{\pi(r)}) \geq 1 - \frac{\epsilon}{2}}
  = \Prob{|\mathrm{Sim}(\bm{x}, \bm{X}_{\pi(r)}) - 1| \leq \frac{\epsilon}{2}}
  \to 1.
\end{align*}
This implies the first statement of the lemma.
The second statement follows from the uniform integrability, since $\bm{x}$ and $\bm{X}_{\pi(1)},\ldots,\bm{X}_{\pi(n)}$ are bounded.
\end{proof}

\begin{proof}[Proof for Lemma \ref{lemma:squaredLossUpperBound}]
Letting
\begin{align*}
\tilde{f}_n(\bm{x}) := \E[Y|\bm{X}]{\hat{f}_n(\bm{x})} = \sum_{i=1}^n V_i(\bm{x})f(\bm{X}_i),
\end{align*}
the squared loss of $\hat{f}_n(\bm{x})$ can be decomposed as the following variance-bias decomposition: for any $\bm{x} \in \mathcal{X}$,
\begin{align*}
\E[\bm{X}, Y]{(\hat{f}_n(\bm{x}) - f(\bm{x}))^2}
  = \E[\bm{X}, Y]{(\hat{f}_n(\bm{x}) - \tilde{f}_n(\bm{x}))^2} + \E[\bm{X}]{(\tilde{f}_n(\bm{x}) - f(\bm{x}))^2}.
\end{align*}
Here, the first and second terms are the variance and bias terms, respectively.
The variance term becomes:
\begin{align*}
\E[\bm{X}, Y]{(\hat{f}_n(\bm{x}) - \tilde{f}_n(\bm{x}))^2}
  &= \E[\bm{X}, Y]{\left(\sum_{i=1}^n V_i(\bm{x})(Y_i - f(\bm{X}_i))\right)^2} \\
&= \begin{aligned}[t]
  &\sum_{i=1}^n \E[\bm{X}]{V_i(\bm{x})^2\CondVar{Y}{\bm{X} = \bm{X}_i}} \\
  &+ 2\sum_{i=1}^n\sum_{j \neq i} \E[\bm{X}]{V_i(\bm{x})V_j(\bm{x})\CondCov{Y_i}{Y_j}{\bm{X}_i, \bm{X}_j}}.
\end{aligned}
\end{align*}
Since $Y_i$ depends only on $\bm{X}_i$,
\begin{align*}
\CondCov{Y_i}{Y_j}{\bm{X}_i, \bm{X}_j} = 0.
\end{align*}
Hence, any covariance term is 0.
Since the number of non-zero $V_i$s is at most $S$, the variance term is bounded as follows:
\begin{align*}
\E[\bm{X}, Y]{(\hat{f}_n(\bm{x}) - \tilde{f}_n(\bm{x}))^2}
  &= \sum_{r=1}^S \E[\bm{X}]{V_{\pi(r)}(\bm{x})^2\CondVar{Y}{\bm{X} = \bm{X}_{\pi(r)}}} \\
&\leq \sigma^2\sum_{r=1}^S \E[\bm{X}]{V_{\pi(r)}(\bm{x})^2}.
\end{align*}
Thus, taking the expectation with $\bm{x}$, we obtain
\begin{align*}
\E{(\hat{f}_n(\bm{X}) - \tilde{f}_n(\bm{X}))^2}
  \leq \sigma^2\sum_{r=1}^S \E{V_{\pi(r)}(\bm{X})^2}.
\end{align*}

The bias term is
\begin{align*}
\E[\bm{X}]{(\tilde{f}_n(\bm{x}) - f(\bm{x}))^2}
  &= \E[\bm{X}]{\left(\sum_{i=1}^n V_i(\bm{x})(f(\bm{X}_i) - f(\bm{x}))\right)^2} \\
&\leq \E[\bm{X}]{\left(\sum_{i=1}^n V_i(\bm{x})|f(\bm{X}_i) - f(\bm{x})|\right)^2} \\
&\leq \E[\bm{X}]{\left(\sum_{i=1}^n V_i(\bm{x}) \cdot L\|\bm{X}_i - \bm{x}\|_2\right)^2}.
\end{align*}
Since $V_i(\bm{x}) \in [0, 1]$ and the number of non-zero $V_i$s is at most $S$, from the Jensen inequality, we have
\begin{align*}
\E[\bm{X}]{(\tilde{f}_n(\bm{x}) - f(\bm{x}))^2}
  &\leq L^2\sum_{i=1}^n\E[\bm{X}]{V_i(\bm{x})\|\bm{X}_i - \bm{x}\|_2^2} \\
  &= L^2\sum_{r=1}^S\E[\bm{X}]{V_{\pi(r)}(\bm{x})\|\bm{X}_{\pi(r)} - \bm{x}\|_2^2}.
\end{align*}
Thus, taking the expectation with $\bm{x}$, we obtain
\begin{align*}
\E{(\tilde{f}_n(\bm{X}) - f(\bm{X}))^2}
  \leq L^2\sum_{r=1}^S\E{V_{\pi(r)}(\bm{X})\|\bm{X}_{\pi(r)} - \bm{X}\|_2^2}.
\end{align*}

Thus, combining the variance and bias terms, we obtain the statement.
\end{proof}


\begin{thebibliography}{99}
\bibitem[Babbar et al.(2017)]{Babbar_etal2017}
  R. Babbar and B. Sch\"olkopf. ``DiSMEC: Distributed Sparse Machines for Extreme Multi-label Classification.'' \textit{Proceedings of the Tenth ACM International Conference on Web Search and Data Mining}, pp. 721--729, 2017.
\bibitem[Bhatia et al.(2015)]{Bhatia_etal2015}
  K. Bhatia, H. Jain, P. Kar, M. Varma, and P. Jain. ``Sparse Local Embeddings for Extreme Multi-Label Classification.'' \textit{Advances in Neural Information Processing Systems}, vol. 28, pp. 730--738, 2015.
\bibitem[Bhatia et al.(2016)]{Bhatia_etal2016}
  K. Bhatia, H. Jain, Y. Prabhu, and M. Varma. The Extreme Classification Repository. 2016. Retrieved January 4, 2018 from \url{http://manikvarma.org/downloads/XC/XMLRepository.html}
\bibitem[Biau et al.(2010)]{Biau_etal2010}
  G. Biau, F. C\'erou, and A. Guyader. ``On the Rate of Convergence of the Bagged Nearest Neighbor Estimate.'' \textit{Journal of Machine Learning Research}, vol. 11, pp. 687--712, 2010.
\bibitem[Biau et al.(2015)]{Biau_etal2015}
  G. Biau, and L. Devroye. \textit{Lectures on the Nearest Neighbor Method}. Springer, 2015.
\bibitem[Bouchard et al.(2013)]{Bouchard_etal2013}
  M. Bouchard, A. Jousselme, and P. Dor\'e. ``A proof for the positive definiteness of the Jaccard index matrix.'' \textit{International Journal of Approximate Reasoning}, vol. 54, no. 5, pp. 615--626, 2013.
\bibitem[Dembczynski et al.(2010)]{Dembczynski_etal2010}
  K. Dembczynski, W. Cheng, and E. H\"ullermeier. ``Bayes Optimal Multilabel Classification via Probabilistic Classifier Chains.'' \textit{Proceedings of the 27th International Conference on Machine Learning}, pp. 279--286, 2010.
\bibitem[Gao et al.(2011)]{Gao_etal2011}
  W. Gao and Z. Zhou. ``On the Consistency of Multi-Label Learning.'' \textit{Proceedings of the 24th Annual Conference on Learning Theory}, vol. 19, pp. 341--358, 2011.
\bibitem[Golang(2009)]{Golang_2009}
  The Go Programming Language, 2009. Retrieved January 4, 2018 from \url{https://golang.org}
\bibitem[Jain et al.(2016)]{Jain_etal2016}
  H. Jain, Y. Prabhu, and M. Varma. ``Extreme Multi-label Loss Functions for Recommendation, Tagging, Ranking \& Other Missing Label Applications.'' \textit{Proceedings of the 22nd ACM SIGKDD International Conference on Knowledge Discovery and Data Mining}, pp. 935--944, 2016.
\bibitem[Lin et al.(2002)]{Lin_etal2002}
  Y. Lin and Y. Jeon. ``Random Forests and Adaptive Nearest Neighbors.'' \textit{Journal of the American Statistical Association}, vol. 101, no. 474, pp. 578--590, 2006.
\bibitem[McMahan et al.(2013)]{McMahan_etal2013}
  H. B. McMahan et al. "Ad Click Prediction: A View from the Trenches." \textit{Proceedings of the 19th ACM SIGKDD International Conference on Knowledge Discovery and Data Mining}, 2013.
\bibitem[Manning et al.(2008)]{Manning_etal2008}
  C. D. Manning, P. Raghavan, and H. Sch\"utze. \textit{Introduction to Information Retrieval}. Cambridge University Press, 2008.
\bibitem[Mohri et al.(2012)]{Mohri_etal2012}
  M. Mohri, A. Rostamizadeh, and A. Talwalkar. \textit{Foundations of Machine Learning}. MIT Press, 2012.
\bibitem[Prabhu et al.(2014)]{Prabhu_etal2014}
  Y. Prabhu, and M. Varma. ``FastXML: A Fast, Accurate and Stable Tree-Classifier for Extreme Multi-Label Learning.'' \textit{Proceedings of the 20th ACM SIGKDD International Conference on Knowledge Discovery and Data Mining}, pp. 263--272, 2014.
\bibitem[Aoshima(2018)]{Aoshima_2018}
  T. Aoshima. Package sticker. 2016. Retrieved February 9, 2018 from \url{https://github.com/hiro4bbh/sticker}
\bibitem[Yen et al.(2017)]{Yen_etal2017}
  I. E. H. Yen, X. Huang, W. Dai, P. Ravikumar, I. Dhillon, and E. Xing. ``PPDSparse: A Parallel Primal-Dual Sparse Method for Extreme Classification.'' \textit{Proceedings of the 23rd ACM SIGKDD International Conference on Knowledge Discovery and Data Mining}, pp. 545--553, 2017.
\bibitem[Yu et al.(2013)]{Yu_etal2013}
  H. Yu, P. Jain, P. Kar, and I. Dhillon, ``Large-scale Multi-label Learning with Missing Labels.'' \textit{Proceedings of the 31th International Conference on Machine Learning}, pp. 325--333, 2013.
\end{thebibliography}
\end{document}